\documentclass{IEEEtran}

\usepackage{amsfonts}
\usepackage{amsgen,amsmath,amstext,amsbsy,amsopn,amssymb,amsthm}
\usepackage{times}
\usepackage{graphicx} % more modern
\usepackage{subfigure} 

\usepackage{natbib}

\usepackage{algorithm}
\usepackage{algorithmic}

\usepackage{hyperref}

\usepackage{float}
\usepackage{bm}
\usepackage{multirow}
\usepackage{subfigure}

\newcommand{\bbR}{\mathbb{R}}

\newcommand{\cL}{\mathcal{L}}

\newcommand{\0}{{\mathbf{0}}}
\newcommand{\1}{{\mathbf{1}}}

\newcommand{\f}{{\mathbf{f}}}

\renewcommand{\v}{{\mathbf{v}}}

\renewcommand{\t}{{\mathbf{t}}}
\newcommand{\p}{\mathbf{p}}

\newcommand{\x}{{\mathbf{x}}}

\newcommand{\z}{{\mathbf{z}}}

\newcommand{\I}{{\mathbf{I}}}

\renewcommand{\H}{{\mathbf{H}}}

\newcommand{\w}{{\mathbf{w}}}

\newcommand{\bPhi}{\mathbf{\Phi}}
\newcommand{\bphi}{\boldsymbol{\phi}}

\newcommand{\bbeta}{\bm{\beta}}

\newcommand{\bmu}{\boldsymbol{\mu}}

\newcommand{\btheta}{\boldsymbol{\theta}}

\newtheorem{theorem}{Theorem}
\newtheorem{proposition}{Proposition}

\title{Nonextensive information theoretical machine}
\author{Chaobing Song, Shu-Tao Xia\\
Graduate School at Shenzhen, Tsinghua University}
\begin{document} 
\maketitle
% \twocolumn[
% \icmltitle{Nonextensive information theoretical machine}

% % It is OKAY to include author information, even for blind
% % submissions: the style file will automatically remove it for you
% % unless you've provided the [accepted] option to the icml2016
% % package.
% \icmlauthor{Chao-Bing Song}{songchaobin@126.com}
% % \icmladdress{Room H205E, Tsinghua Park, University Town of XiLi, NanShan District, ShenZhen City, GuangDong Province, China}
% \icmlauthor{Shu-Tao Xia}{xiast@sz.tsinghua.edu.cn}
% \icmladdress{Tsinghua Park, University Town of XiLi, NanShan District, ShenZhen City, GuangDong Province, China}

% % You may provide any keywords that you 
% % find helpful for describing your paper; these are used to populate 
% % the "keywords" metadata in the PDF but will not be shown in the document
% \icmlkeywords{information theory, machine learning, binary classification}
% \vskip 0.3in
% ]

\begin{abstract} 
In this paper, we propose a new discriminative model named \emph{nonextensive information theoretical machine (NITM)} based 
on nonextensive generalization of Shannon information theory. 
In NITM, weight parameters are treated as random variables. 
Tsallis divergence is used to regularize the distribution of weight parameters and maximum unnormalized
Tsallis entropy distribution is used to evaluate fitting effect.  On the one hand, it is showed that
some well-known margin-based loss functions such as $\ell_{0/1}$ loss,  hinge loss, squared hinge loss and exponential loss can be unified by unnormalized Tsallis entropy. On the other hand, Gaussian prior regularization is generalized to Student-t prior regularization with similar computational complexity.
The model can be solved efficiently by gradient-based convex optimization and its performance is illustrated on standard datasets.
\end{abstract} 

\section{Introduction}
\label{intro}
As the representatives of statistical learning and ensemble learning methods respectively, support vector machine (SVM) \cite{cortes1995support} and adaboost \cite{freund1997decision}  have got a lot of success in practice. They can both be classified in the margin-based classification methodology \cite{rosset2004boosting}. From the view of loss function, 
in SVM, hinge loss is employed as measure to find the maximum margin plane. While in adaboost, exponential loss is used to select and combine weak learners. In terms of regularization, $\ell_2$-norm and $\ell_1$-norm corresponds to Gaussian prior and Laplace prior \cite{zhu2009maximum} respectively 
% from probabilistic perspective 
and are often used to control the model complexity of SVM.
While in the boosting framework,   iterative regularization is often used as approximate $\ell_1$ regularization \cite{rosset2004boosting}. In terms of data transform, SVM maps data into high dimension by kernel function, while adaboost transforms data as the output of weak learners. 

Two interesting questions are whether we can unify the mathematical form of SVM and adaboost in a common framework and whether loss function, regularization method and data transform method can be expressed by a unified mathematical theory. In this paper, we give an  attempt under nonextensive information theory (NIT) framework. 
In complex systems with long-range interaction, long-time memory and multifractals \cite{tsallis2001nonextensive}, the equilibrium state often shows power-law distribution instead of exponential distribution. Therefore, the well-known Boltzmann distribution (which is exponential distribution) cannot be well used. NIT as a generalization of Shannon information theory aims to model power-law phenomenon by generalizing Boltzmann-Gibbs-Shannon (BGS) entropy to Tsallis entropy of which the maximum entropy distribution is power-law distribution if the entropy index $q\neq 1$.

In machine learning, there has been some applications of Tsallis entropy and its related concepts 
% such as Tsallis divergence, $q$-logarithm and $q$-exponent,
 such as Tsallis mutual information kernel \cite{martins2009nonextensive}, t-logistic regression \cite{ding2010t}, approximate inference based on t-divergence \cite{ding2011t}. In \cite{martins2009nonextensive}, Tsallis mutual information kernel is proposed by extending Jensen-Shannon divergence and Shannon entropy to Jensen-Tsallis $q$-difference and Tsallis entropy; in \cite{ding2010t}, convex loss is extended to nonconvex loss by using $q$-exponential families; in \cite{ding2011t}, approximate inference is used to $q$-exponential family by defining a new divergence. 

% In this paper, we propose nonextentive information theoretical machine (NITM) based on NIT. NITM aims to address
% address the binary classification task by . 
Concretely, our contributions are:
\begin{itemize}
\item By using the concepts and methods from NIT, we propose nonextentive information theoretical machine (NITM) to address binary classification task. Its solution and explicit primal and dual formulations are given.
\item By observation, we show that all the well-known $\ell_{0/1}$ loss, hinge loss, squared hinge loss and exponential loss are the maximum unnormalized Tsallis entropy distribution with different entropy indices $q$;
\item By using Tsallis divergence and $q$-expectation, we show that Gaussian prior ($\ell_2$ norm) regularization can be extended to the more general Student-t prior regularization with similar computational complexity.
\item By considering the existing work of nonextensive mutual information kernel \cite{martins2009nonextensive}, we show that all the three parts of discriminative model, e.g., loss function, regularization and data transform can be expressed consistently under the framework of NIT.
\item By experiments, it is showed that NITM can improve the generalization performance on different standard datasets by tuning entropy indices properly.
% Tsallis divergence is used to regularize the distribution of weight parameters and $q$-expectation is used to constrain the mean of weight parameters. By these, the posterior distribution of weight parameters will have the same form if the prior distribution is Student-t distribution with specified degrees of freedom.
\end{itemize}
\section{Nonextensive information theory}\label{sec:preliminaries}
Nonextensive information theory (NIT) has raised a lot of interest in physical community. In this section, we mainly review some necessary concepts from NIT.

For convenience, firstly $q$-exponent and $q$-logarithm \cite{tsallis2001nonextensive} are defined as 
\begin{eqnarray*}
\exp_{q}x&=&
\begin{cases}
(1+(1-q)x)_+^{\frac{1}{1-q}}, & q\in\bbR\backslash\{1\}\\
\exp x,    & q=1 \\
\end{cases},  \\
\ln_{q}x&=&\begin{cases}
	\frac{x^{1-q}-1}{1-q}, &  q\in\bbR\backslash\{1\}\\
	\ln x, & q=1
\end{cases},
\end{eqnarray*}
where $[x]_+$ stands for $\max\{x,0\}$ and $\exp_1 x=\lim_{q\rightarrow 1}\exp_q x=\exp x$, $\ln_1x=\lim_{q\rightarrow1}\ln_q x=\ln x$.
 By its definition, one has
\begin{eqnarray*}
\exp_{q}(\ln_{q}x)&=&x,\\
\ln_{q}(\exp_{q}x)&=&x.
\end{eqnarray*}

Corresponding to the definition of exponential family, one can define $q$-exponential family \cite{amari2011geometry} as
\begin{eqnarray*}
	p(\x;\btheta)=\exp_q(\btheta^T x-\psi_q(\btheta)),
\end{eqnarray*}
where $\btheta$ is parameters and $\psi_q(\btheta)$ is log normalized factor.

In addition, denote indicator function 
\begin{eqnarray*}
I_{\infty}(A)=
\begin{cases}
\infty, & \text{event}\; A \;\text{holds} \\
0, & \text{else}
\end{cases}.
\end{eqnarray*}
Denote the real line and the nonnegative half-line by $\bbR$ and $\bbR_+$ respectively. The set of $n$-dimensional vectors with positive components of sum $1$ is denoted by
\begin{equation*}
\Delta_n=\left\{\v,\v=(v_1,v_2,\ldots,v_n)^T\in\bbR_+^n, \sum_{i=1}^{n}v_i=1\right\}.
\end{equation*}
% For a variable $\v\in\bbR_+^n$ and $q>0$, denote $\v^q=(\v_1^q,\v_2^q,\ldots,\v_n^q)^T$. For $q\rightarrow+\infty$, denote $\v^{\infty}=\lim_{q\rightarrow+\infty}\v^q=(I_{\infty}(\v_1>1),I_{\infty}(\v_2>1),\ldots,I_{\infty}(\v_n>1))^T$. 
In addition, denote $\mathbf{1}_n$ as a vector in $\bbR_+^n$ with all elements $1$.

For $\p\in\Delta_n$, Tsallis entropy is defined as \cite{tsallis1988possible,tsallis2001nonextensive} 
\begin{eqnarray*}
S_{q}(\p)&=&k\sum_{i=1}^{n}p_i\ln_q \frac{1}{p_i} \\
&=&\begin{cases}
-k\frac{\sum_{i=1}^{n}p_i^q-1}{q-1},\quad& q\in\bbR\backslash\{1\}	\\
-k\sum_{i=1}^{n}p_i\ln p_i, \quad & q=1
\end{cases},
\end{eqnarray*}
where $k$ is an arbitrary positive constant. For convenience, set $k=1$ in the following context.  
For $q=1$, $S_1(\p)$ is equivalent to the definition of Shannon entropy.
For $q=0$ and $i\in\{1,2,\ldots,n\}$, define $p_i^q=0$ if $p_i=0$ and $p_i^q=1$ if $p_i\neq 0$, then
\begin{equation*}
S_0(\p)=\|\p\|_0-1,
\end{equation*}
where $\|\cdot\|_0$, called $\ell_0$ pseudo norm, denotes the number of nonzero elements in vector. 
If $q<0$, $S_q(p)$ is convex; if $q>0$, $S_q(p)$ is concave. In all cases, $S_{q}\ge0$ (nonnegativity property). 
For two independent random variables $A$ and $B$, with probability mass function $\p_A\in\Delta_{n_A}$ and $\p_B\in\Delta_{n_B}$ respectively, consider the new random variable $A\cup B$ defined by the joint distribution $\p_A\cup \p_B\in\Delta^{n_A n_B}$, then \cite{tsallis1988possible},
\begin{equation*}
S_{q}(\p_A\cup \p_B)=S_{q}(\p_A)+S_{q}(\p_B)+(1-q)S_{q}(\p_A)S_{q}(\p_B),
\end{equation*}
which is called the nonextensive property of Tsallis entropy. 
One can immediately see that $q< 1, q = 1$ and $q >1$ respectively correspond to superextensivity (superadditivity), extensivity (additivity) and subextensivity (subadditivity).   An axiomatic framework for Tsallis entropy (for all $q\in\bbR$) and an uniqueness theorem can be seen in
\cite{dos1997generalization}.
 
As a measure of similarity on $\bbR_+^n$, for $\p,\t\in \bbR_+^n$ and $q\in \bbR$, generalized Tsallis divergence \cite{martins2009nonextensive} is defined as
 \begin{eqnarray}
 &&D_{q}(\p\|\t)
 =\sum_{i=1}^n -p_i\ln_q\left(\frac{t_i}{p_i}\right)-p_i+t_i\nonumber\\
 &=&
\begin{cases}
\frac{\sum_{i=1}^{n}p_i^{q}t_i^{1-q}-q p_i+(q-1) t_i}{q-1}\;, & q\in\bbR\backslash\{1\} \\
\sum_{i=1}^{n}p_i\ln\frac{p_i}{t_i}-p_i+t_i\;, & q=1
\end{cases}\label{eq:gene-Dq}
 \end{eqnarray}
For $q=1$, $D_{q}(\p\|\t)$  is  the definition of the generalized Kullback-Leibler (KL) divergence \cite{csiszar1975divergence}.

For the case $\p,\t\in\Delta_n$, by the definition of $D_{q}(\p\|\t)$ in 
\eqref{eq:gene-Dq}, one has 
\begin{eqnarray*}
 D_{q}(\p\|\t)&=&\sum_{i=1}^n -p_i\ln_q\left(\frac{t_i}{p_i}\right)\\
 &=&
\begin{cases}
\frac{\sum_{i=1}^{n}p_i^{q}t_i^{1-q}-1}{q-1}\;, & q\in\bbR\backslash\{1\} \\
\sum_{i=1}^{n}p_i\ln\frac{p_i}{t_i}\;, & q=1
\end{cases},
 \end{eqnarray*}
which is called the Tsallis divergence on discrete probability distribution. For $q=1$, $D_1(\p\|\t)$ is the well-known KL divergence.

Similarly, for two \emph{unnormalized} probability density functions (pdf) $p(\x)$ and $t(\x)$ on $\x\in\bbR^n$, the generalized Tsallis divergence  can be defined as
\begin{eqnarray*}
&& D_{q}(p(\x)\|t(\x))\\
 &=&\int\left(-p(\x)\ln_q\left(\frac{t(\x)}{p(\x)}\right)-p(\x)+t(\x)\right)d\x\\ 
&=&
\begin{cases}
\frac{\int\left( p^{q}(\x)t^{1-q}(\x) -q p(\x)+(q-1)t(\x)\right)d\x}{q-1}, & q\in\bbR\backslash\{1\} \\
\int \left(p(\x)\ln\frac{p(\x)}{t(\x)}-p(\x)+t(\x)\right)d\x, & q = 1
\end{cases}.
\end{eqnarray*}
For normalized pdfs $p(\x)$ and $t(\x)$, where $\int p(\x)d\x=1, \int t(\x)d\x=1$, one has
\begin{eqnarray*}
D_{q}(p(\x)\|t(\x))
&=&\int -p(\x)\ln_q\left(\frac{t(\x)}{p(\x)}\right)d\x\\
&=&
\begin{cases}
\frac{\int p^{q}(\x)t^{1-q}(\x)d\x-1}{q-1}, & q\in\bbR\backslash\{1\} \\
\int p(\x)\ln\frac{p(\x)}{t(\x)}d\x, & q = 1
\end{cases},
\end{eqnarray*}
which is called the Tsallis divergence on continuous probability distribution.
Meanwhile, for the normalized pdf $p(\x)$, Tsallis entropy can be defined as 
\begin{eqnarray*}
S_{q}(p(\x))=-\frac{\int p^{q}(\x)d\x-1}{q-1}. %(q\in \bbR)
\end{eqnarray*}
 
For $q>0$, $D_q(p(\x)\|t(\x))$ is a special case of f-divergence (see \cite{cichocki2010families} and reference therein), which has the following properties.
\begin{itemize}
\item Convexity: $D_q(p(\x)\|t(\x))$ is convex with respect to (\emph{w.r.t.}) both $p(\x)$ and $t(\x)$;
\item Strict Positivity: $D_q(p(\x)\|t(\x))\ge0$ and $D_q(p(\x)\|t(\x))=0$ if and only if $p(\x)=t(\x)$.
\end{itemize}
Because of the two useful properties, the value of $D_q(p(\x)\|t(\x))$ with $q>0$ can be used to measure the similarity between $p(\x)$  and $t(\x)$. In practice,
one can make $p(\x)$ get close to $t(\x)$ as much as possible by minimizing $D_q(p(\x)\|t(\x))$ \emph{w.r.t.} $p(\x)$. 

The above two properties also hold for $D_q(\p\|\t)$ in the discrete case.

Particularly, in the discrete case, let $\t=\frac{1}{n}\mathbf{1}_n$, then for $\p\in\Delta_n$,
\begin{eqnarray*}
&&D_q\left(\p\|\frac{1}{n}\mathbf{1}_n\right)
=\sum_{i=1}^n -p_i\ln_q\left(\frac{\frac{1}{n}\mathbf{1}_n}{p_i}\right)\\
&=&
\begin{cases}
 n^{q-1}\frac{\sum_{i=1}^n p_i^q-1}{q-1}-\frac{1-n^{q-1}}{q-1},& q\in\bbR\backslash\{1\} \\
 \sum_{i=1}^{n}p_i\ln p_i+\ln n, & q=1
\end{cases}\\
&=&\begin{cases}
 -n^{q-1}S_q(\p)-\frac{1-n^{q-1}}{q-1},& q\in\bbR\backslash\{1\} \\
 -S_1(\p)+\ln n, & q=1
\end{cases},
\end{eqnarray*}
which shows that for a fixed $q$, there exists a one-to-one correspondence between $D_q\left(\p\|\frac{1}{n}\mathbf{1}_n\right)$ and $S_q(\p)$. In fact, the entropy of $\p$ can be understood as the degree of similarity from $\p$ to uniform distribution \cite{shore1980axiomatic}. Therefore, maximizing Tsallis entropy $S_q(\p)$ is equivalent to minimizing Tsallis divergence $D_q\left(\p\|\frac{1}{n}\mathbf{1}_n\right)$.
% From this derivation, Tsallis entropy $S_q(\p)$ is linear dependent on $D_q\left(p\|\frac{1}{n}\mathbf{1_n}\right)$ and the dependent coefficient is less than $0$. Therefore, 
% for a fixed $q$, the Tsallis divergence from $p$ to uniform distribution can be used to measure the Tsallis entropy of $p$. Based on this fact, minimizing Tsallis divergence $D_q(p\|t)$ can be seen as a generalization of maximizing Tsallis entropy $S_q(p)$  

For unnormalized discrete probability distribution $\p$ and $q>0$, $D_q\left(\p\|\mathbf{1}_n\right)$ is also an effective measure to the distance from $\p$ to the unnormalized uniform distribution $\mathbf{1}_n$. Neglecting constants, one can define $-D_q\left(\p\|\mathbf{1}_n\right)$  as the \emph{unnormalized Tsallis entropy} of $\p$.
Therefore,  minimizing $D_q\left(\p\|\mathbf{1}_n\right)$ can be seen as maximizing unnormalized Tsallis entropy of $\p$. In order to describe the result in Section \ref{sec:NITM} consistently, we define $D_{\infty}\left(\p\|\mathbf{1}_n\right)$ by its limit given by
\begin{eqnarray}
D_{\infty}\left(\p\|\mathbf{1}_n\right)&=&\lim_{q\rightarrow+\infty} D_q\left(\p\|\mathbf{1}_n\right)\nonumber\\
&=&\sum_{i=1}^n -p_i+I_{\infty}(p_i\le 1)+n.\label{eq:D-infty}
\end{eqnarray}

\section{Nonextensive information theoretical machine}\label{sec:NITM}
% Nonextensive information theoretical machine (NITM) is a very general model which aims to unify  regularization and loss function based on a common framework.  Similar to MaxEnDNet \cite{zhu2009maximum}, NITM combines Bayesian prior and loss function in one model. Unlike MaxEnDnet, 

% % It explains the fact that minimizing the two relatively independent parts is both equivalent to maximizing some generalized Tsallis entropy. Although a lot of literature have dedicated to connect machine learning to information theory, such a consistent view has not been provided to the best of our knowledge.

% For simplicity, we mainly discuss the results in linear classifier setting. Feature map  are relatively independent parts and can be integrated by just changing the input data from the original data $\{\x_1,\x_2,\ldots,\x_m\}$ to feature data $\{\f_1,\f_2,\ldots,\f_m\}$. Kernel trick can also be used in our model if we solve NITM in the dual formulation. 

% the common nonextensive information theoretical framework. It explains a very interesting phenomenon, 

% We discuss our results in linear classifier setting. Our model mainly consists of regularization term and loss function, kernel is relatively an independent part and integrated 
% without changing the existing model. 

Given a set of instance-label pairs $(\x_i, y_i),$ $i\in\{1,2,\ldots,m\}$, $\x_i\in\bbR^n,y_i\in\{-1,+1\}$, $\{\phi_i(\cdot)\}_{i=1}^d$ is a group of fixed basis functions. Denote $\bPhi=(\bphi_1,\bphi_2,\ldots,\bphi_d)=(\f_1,\f_2,\ldots,\f_m)^T$, where $\bphi_j=(\phi_j(\x_1),\phi_j(\x_2),\ldots,\phi_j(\x_m))^T$ for $j=1,2,\ldots,d$ and $\f_i=(\phi_1(\x_i),\phi_2(\x_i),\ldots,\phi_d(\x_i))^T$ for $i=1,2,\ldots,m$.
 Nonextensive information theoretical machine (NITM) solves the following constrained problem:
\begin{eqnarray}
\min_{p(\w),\z}&&\!\!\!\!\!\!\!\!D_q(p(\w)\|p_0(\w))+C\sum_{i=1}^{m}\exp_{q^{\prime}}(-z_i)  \label{eq:model}\\
%\label{eq:master}\\
s.t. && \!\!\!\!\!\!\!\! z_i=\int y_i\f_i^T\w p^q(\w)d\w,\quad i=1,2, \ldots , m \label{eq:q-expectation}, \\
&&\!\!\!\!\!\!\!\!\int p(\w)d\w=1, \nonumber
%&&\!\!\!\!\!\!\!\! p(\w)\ge0,\nonumber
\end{eqnarray}
where $\w\in\bbR^d$ is assumed to be a continuous random vector with normalized pdf $p(\w)$. 
% The following derivation will show that the use of $q$-expectation is vital to get a pdf $p(\w)$ having the same form with the prior distribution $p_0(\w)$.  
Unlike the common $\ell_2$-norm or $\ell_1$-norm regularization, we impose Bayesian prior $p_0(\w)$ on $\w$ and use Tsallis divergence 
$$D_q(p(\w)\|p_0(\w))=\frac{\int p^q(\w) p_0^{1-q}(\w)d\w-1}{q-1}$$
 to measure the distance of distribution from the posterior distribution $p(\w)$ to $p_0(\w)$. Instead of using the normal expectation \cite{zhu2009maximum},  $q$-expectation $\int\w p^{q}(\w)d\w$ in \eqref{eq:q-expectation} is used \cite{curado1991generalized}.
 Meanwhile,
\begin{eqnarray}
\exp_{q^{\prime}}(-z_i)&=&[1-(1-q^{\prime})z_i]_{+}^{\frac{1}{1-q^{\prime}}}\label{eq:p_prime}
\end{eqnarray}
can be seen as an unnormalized probability mass distribution (pmf) belonging to $q^{\prime}$-exponential family. 
%In the following derivation, it will be showed that minimizing the parameters of an unnormalized pmf is equivalent to minimizing the corresponding  generalized  divergence  to some reference vector. 
The sum $\sum_{i=1}^{m} \exp_{q^{\prime}}(-z_i)$ is used as loss function. The regularization term and loss function are connected by the constraint of  $q$-expectation (\ref{eq:q-expectation}). $C>0$ is the regularization parameter to tune the relative weight of the two terms.  $q$ and $q^{\prime}$ are called ``entropy indices'' in NIT.
% In addition,
% According to margin-based learning $z_i$ is called the margin of $\x_i$.

Due to the Bayesian-style treatment of $\w$, the final output used to give a discriminant to a new data $\x$ is the posteriori $q$-expectation, denoted as  
\begin{eqnarray*}
\langle\w\rangle_{p^{q}}=\int \w^Tp^{q}(\w)d\w,
\end{eqnarray*}
and the discriminative function is 
\begin{eqnarray*}
y(\x)=\arg\max_{y\in\{-1,1\}} y\cdot\x^T\langle\w\rangle_{p^{q}}.
\end{eqnarray*}
It should be noted that $\langle\w\rangle_{p^{q}}$ is needed to exist in this paper, but it does not mean the normal expectation $\langle\w\rangle_p$ exists at the same time.

Setting $q^{\prime}$ to $\{0,1/2\}$ and taking limit at $q^{\prime}\rightarrow-\infty,1$, one has the following result.
\begin{theorem}\label{thm:0}
The well-known $\ell_{0/1}$ loss, hinge loss, squared hinge loss and exponential loss can be unified in $q^{\prime}$-exponential family. The corresponding relation with $q^{\prime}$ can be seen in Table \ref{tb:tb1}.
\begin{table}[ht]
\center
\caption{Loss functions with specified $q^{\prime}$}
\begin{tabular}{|c|c|c|}
\hline
$q^{\prime}$ & $\exp_{q^{\prime}}(-z)$ &Notes  \\
\hline
$-\infty$ & $I(z<0)$ & $\ell_{0/1}$ loss\\
$0$ & $[1-z]_+$ & hinge loss \\
$\frac{1}{2}$ &  $[1-\frac{1}{2}z]_+^2$ & squared hinge loss \\
$1$ & $\exp(-z)$ & exponential loss \\
\hline
\end{tabular}\label{tb:tb1}
\end{table}
\end{theorem}
\begin{proof}
The proof for $q^{\prime}=0,\frac{1}{2},1$ is neglected. 

For $q^{\prime}\rightarrow -\infty$, 
if $z=0$, then $\exp_{q^{\prime}}(z)=1$; if $z>0$, $[1-(1-q^{\prime})z]_+=0$, thus $\exp_{q^{\prime}}(z)=0$; if $z<0$, 
\begin{eqnarray}
&&\lim_{q^{\prime}\rightarrow-\infty}\ln\exp_{q^{\prime}}(z)\nonumber\\
&=&\lim_{q^{\prime}\rightarrow-\infty}\frac{\ln(1+(1-q^{\prime})z)}{1-q^{\prime}}\nonumber\\
&=&\lim_{q^{\prime}\rightarrow-\infty}\frac{z}{1+(1-q^{\prime})z}=0.\nonumber
\end{eqnarray}
Therefore, if $z<0$, $\lim_{q^{\prime}\rightarrow-\infty}\exp_{q^{\prime}}(z)=1$. 
\end{proof}

%From Table \ref{tb:tb1}, the well-known $\ell_{0/1}$ loss, hinge loss, squared hinge loss and exponential loss are all the special cases to the $q^{\prime}$-exponent $\exp_{q^{\prime}}(-z)$.
From Theorem \ref{thm:0}, $\ell_{0/1}$ loss corresponds to  $q^{\prime}$-exponential family with $q^{\prime}\rightarrow -\infty$, which is concave.
Hinge loss can be seen as the tightest convex relaxation to $\ell_{0/1}$ loss, which is similar to the relationship between $\ell_1$-norm and $\ell_0$-norm. For $q^{\prime}=\frac{1}{2}$, the coefficient $\frac{1}{2}$ is only a scale factor and the formulation is equivalent to the standard squared hinge loss $[1-z]_+^2$ after scaling $z$. For $q^{\prime}>1$, as $z\rightarrow (\frac{1}{1-q^{\prime}})^{+}$, $\exp_{q^{\prime}}(-z)\rightarrow+\infty$.  Then if one wants the objective function is bounded in any bounded interval, $q^{\prime}=1$, which corresponds to exponential loss, is the largest value we can choose. Therefore, in this paper, $q^{\prime}$ is selected in $[0,1]$. 

%In order to make NITM work, one should select the value of $q$ and the form of $p_0(\w)$. While the general model doesn't not constrain the selection, one stiil 

The general model doesn't constrain the selection of $q$ and $p_0(\w)$, but it is necessary to select them carefully for  model effectiveness and computational efficiency. In this paper, Student-t distribution is considered, for its good properties.
\begin{itemize}
\item Its support is $\bbR^{d}$;
\item By varying its degrees of freedom $\nu$, it can model the heavy tailed distribution with different thickness; 
\item Taking $\nu\rightarrow +\infty$, it is equivalent to Gaussian distribution;
% \item If the $q$-expectation and the normal expectation both exist, they are proportional;
% \item Combining a Student-t prior distribution, the corresponding Tsallis divergence and $q$-expectation in NITM, the posterior distribution $p(\w)$ will also be a Student-t distribution:
\end{itemize} 

The general model \eqref{eq:model} couples a variational optimization subproblem and a numerical optimization subproblem together. For $q\ge1$, $D_q(p(\w)\|p_0(\w))$ in \eqref{eq:model} and $q$-expectation in \eqref{eq:q-expectation} are convex \emph{w.r.t.} $p(\w)$.
 In addition, for $0\le q^{\prime}\le1$, $\exp_{q^{\prime}}(-z_i)$ in \eqref{eq:p_prime} is also convex. Therefore, for the entropy indices $q\ge1$ and $0\le q^{\prime}\le1$, the general model is a convex problem \emph{w.r.t.} $p(\w)$ and $\z$. On the one hand,  the problem can be solved directly by some variational optimization technique, or convex optimization method if $D_q(p(\w)\|p_0(\w))$ and $q$-expectation can be explicitly expressed in terms of distribution parameters. On the other hand, one can solve it indirectly by solving the Lagrange dual problem. Our first main result is about the solution of $p(\w)$ expressed by Lagrange multipliers and the dual optimization formulation of the general model \eqref{eq:model}.

% By imposing a normalizable prior distribution $p_0(\w)$,  
%    Tsallis divergence and $q^{\prime}$-exponent are both convex and the constraints are 

% Our main results are presented as follows. On the one hand, by solving the general model using Lagrange multiplier

% Our main results are about the solution of the posterior distribution $p(\w)$ and th
 
%Consider the feature of data, in practice, one may tune $q^{\prime}$ to get a better generalization performance.

% Denote
% \begin{eqnarray}
%  \sum_{i=1}^{m}\beta_i(-y_i\x_i^T\w)&=&\bbeta^T H\w \\
% \c &=& H^T\bbeta
% \end{eqnarray}
% Then

\begin{theorem}\label{thm:1}
For $q\ge1$ and $0\le q^{\prime}\le1$,
the posterior distribution $p(\w)$ of the general problem \eqref{eq:model} can be expressed in terms of the prior distribution $p_0(\w)$ and the Lagrange multipliers as 
\begin{eqnarray}
p(\w)=
\frac{1}{Z_q(\bbeta)}p_0(\w)\exp_q(p_0^{q-1}(\w)\bbeta^T \H\w), \label{eq:pw0}
\end{eqnarray}
% \begin{eqnarray*}
% p(\w)=
% \begin{cases}
% \frac{1}{Z_q(\bbeta)}\left[p_0^{1-q}(\w)+(1-q)\bbeta^T H\w\right]_+^{\frac{1}{1-q}}, & q>1 \\
% \frac{1}{Z_1(\bbeta)}p_0(\w)\exp(\bbeta^T H\w), & q=1
% \end{cases}
% \end{eqnarray*}
where $Z_q(\bbeta)$ is a normalizable factor, $\bbeta$ is the Lagrange multipliers and $\H=(y_1\f_1, y_2\f_2,\ldots, y_m\f_m)^T$. 
% For $q\rightarrow 1^+$, 
% \begin{eqnarray*}
% p(\w)=\frac{1}{Z_q(\bbeta)}p_0(\w)\exp(-\bbeta^T H\w).
% \end{eqnarray*}

Meanwhile, one can solve the primal problem in the dual domain by optimizing the following formulation  
\begin{eqnarray}\label{eq:dual1}
\min_{\bbeta} &&\ln_{q}(Z_q(\bbeta))+CD_{1/q^{\prime}}\left(\bbeta/C\|\mathbf{1}_m\right)\\
% \frac{\sum_{i=1}^{m}(q^{\prime}(\frac{\bbeta_i}{C})^{\frac{1}{q^{\prime}}}-\frac{\bbeta_i}{C})}{q^{\prime}-1}\nonumber\\
s.t.  &&\bbeta\ge\0. \nonumber
\end{eqnarray}
%  where $\mathbf{1}_m$ is a vector in $\bbR^m$ with all elements $1$. For $q^{\prime}\rightarrow 0^{+}$, the formulation is 
%  \begin{eqnarray*}
% \max_{\bbeta} &&-\log_{q}(Z_q(\bbeta))+\sum_{i=1}^{m}\bbeta_i\\
% s.t.  &&0\le\bbeta\le C.
% \end{eqnarray*}
% For $q^{\prime}\rightarrow 1^{-}$, it is
%  \begin{eqnarray*}
% \max_{\bbeta} &&-\log_{q}(Z_q(\bbeta))+\sum_{i=1}^{m}C(\frac{\beta_i}{C}-\frac{\beta_i}{C}\ln\frac{\beta}{C})\\
% s.t.  &&\bbeta\ge 0.
% \end{eqnarray*}
\end{theorem}
The posterior distribution $p(\w)$ in \eqref{eq:pw0} is parametrized by dual variables $\bbeta$. The factor $p_0^{q-1}(\w)$ in $\exp_q(\cdot)$ is emerged by the use of $q$-expectation, which is the key to get a normalizable solution of $p(\w)$.  In \eqref{eq:dual1},  it shows that minimizing the sum of the unnormalized pmf $\sum_{i=1}^{m} \exp_{q^{\prime}}(-z_i)$ \emph{w.r.t.} $\z$ under the constraint \eqref{eq:q-expectation} is equivalent to maximizing the unnormalized Tsallis entropy $-D_{1/q^{\prime}}\left(\bbeta/C\|\mathbf{1}_m\right)$ of the scaled dual variables $\frac{\bbeta}{C}$ under the nonnegative constraint $\bbeta>\0$. For $q^{\prime}\rightarrow 0^{+}$, \emph{i.e.}, $1/q^{\prime}\rightarrow+\infty$, according to \eqref{eq:D-infty}, 
\begin{eqnarray}
D_{\infty}\left(\bbeta/C\|\mathbf{1}_m\right)=\sum_{i=1}^m -\frac{\beta_i}{C}+I_{\infty}\left(\frac{\beta_i}{C}\le1\right)+m,
\end{eqnarray}
which is equivalent to the dual formulation of hinge loss \cite{zhu2009maximum}.

In \cite{zhu2009maximum}, the authors emphasize the advantage of combing maximum entropy learning with maximum margin learning. However, from our perspective, maximum margin learning is the dual formulation of the maximum unnormalized Tsallis entropy learning. Therefore, maximum entropy learning and maximum margin learning  can be unified by the concepts of NIT in the NITM model.

Consider the Student-t prior distribution
% Our optimization model \eqref{eq:master} is quite general. In order to make it work, at least one should designate some concrete form to the priori distribution $p_0(\w)$. In this paper, Student-t distribution is considered, for its good properties: firstly, its support is $R^{n}$; secondly, by varying its degree of freedom $\nu$, it can model the heavy tailed distribution with different thickness,; thirdly,  taking $\nu\rightarrow +\infty$, it is equivalent to Gaussian distribution;  fourthly, the mean and covariance are easy to compute, and if the $q$-expectation and the standard expectation both exist, they are colinear. In our setting,
\begin{eqnarray}
p_0(\w)&=&\frac{1}{Z_{0}}\left(1
+\frac{1}{\nu}\|\w\|_2^2\right)^{-\frac{\nu+d}{2}}, \label{eq:p0nu}
\end{eqnarray}
where $Z_{0}=\frac{\Gamma(\nu/2)\nu^{d/2}\pi^{d/2}}{\Gamma((\nu+d)/2)}$, $d$ is the dimension of $\w$, $\nu>0$ is the degrees of freedom. $\Gamma(\cdot)$ denotes Gamma function. 
For $\nu>2$, both the mean and covariance of $p_0(\w)$ exist and equal $\0$ and $\frac{\nu}{\nu-2}\I$ respectively.

% In order to explain the role of parameters, we give the following definition.

% \begin{definition}\label{def:1}
% The parameters of distribution are called \emph{linear parameters} if changing the parameters is equivalent to scaling, translating or rotating the random variables; otherwise they are called \emph{nonlinear parameters}.
% \end{definition}
% By Definition \ref{def:1}, the mean and covariance of $p_0(\w)$ are linear parameters and the degrees of freedom $\nu$ is nonlinear parameter. By setting the mean and covariance of $p_0(\w)$ to $\0$ and $\I$ respectively, we mainly pay attention to the nonlinear change by $\nu$.

% According to whether the pdf curve can be recovered by scale and shift transform \emph{w.r.t.} random variables, the type of pdf curve change can be divided into linear change and nonlinear change. Changing the mean or covariance belongs to linear change, while changing the degrees of freedom $\nu$ is some kind of \emph{nonlinear change}. In this paper, we mainly interested in the nonlinear change caused by $\nu$, so this form of Student-t distribution is selected. 

 % In addition, we assume the covariance exists, then $\nu>2$.
In order to get an analytic solution, we set 
\begin{eqnarray}
\frac{1}{q-1}=\frac{\nu+d}{2}>\frac{d}{2}, \nonumber
\end{eqnarray}
then
\begin{equation}
q<\frac{2+d}{d}.
\end{equation}
% With this setting, $p(\w)$ is also a Student-t distribution and $\langle \w\rangle_q$ exists.
% In order to let $\langle \w\rangle_q$ exist, by the property of Student-t distribution, $\frac{q}{q-1}>\frac{n+1}{2}$ is required, i.e., $q<\frac{n+1}{n-1}$.
In addition, if $\nu$ is expressed by $q$, then
\begin{equation}
p_0(\w)
=\frac{1}{Z_{0}}\left(1
+\frac{q-1}{2-d(q-1)}\|\w\|_2^2\right)^{\frac{1}{1-q}}, \label{eq:prior}
\end{equation}
where $Z_0$ can be written as
\begin{equation}
 Z_{0}=\frac{\Gamma\left(\frac{2-d(q-1)}{2(q-1)}\right)\left(\frac{2-d(q-1)}{q-1}\pi\right)^{\frac{d}{2}}}{\Gamma\left(\frac{1}{q-1}\right)}.	\label{eq:Z}
 \end{equation} 

% where the partition function $Z_{0}=\frac{\Gamma\left(\frac{2+n-nq}{2(q-1)}\right)\pi^{n/2}(2+n-nq)^{1/2}}{\Gamma\left(\frac{1}{q-1}\right)(q-1)^{n/2}}$.

% Using $v$ instead of $q$ as a parameter, one can get more simplistic expression of results. Therefore, in the following statement, we express the results about the model with Student-t prior using the parameter $v$. 

Imposing the above prior distribution $p_0(\w)$, the normalization factor $Z_q(\bbeta)$ can be expressed explicitly.  Thus one has the following concrete results. 
\begin{theorem}\label{thm:2}
Assume $1\le q<\frac{2+d}{d}$ and $0\le q^{\prime}\le1$, $\frac{1}{q-1}=\frac{\nu+d}{2}$ and $p_0(\w)$
is given in \eqref{eq:prior}.  
Then the posterior distribution $p(\w)$ of the general problem \eqref{eq:model} can be expressed in terms of the prior distribution $p_0(\w)$ and the Lagrange multipliers as 
\begin{eqnarray}
p(\w)=\frac{1}{Z_{0}c^{d/2}}\left(1+\frac{1}{\nu c}\|\w-\bmu\|_2^2\right)^{-\frac{\nu+d}{2}}, \label{eq:pw}
\end{eqnarray}
where 
\begin{equation}
\begin{split}
\bmu&=\frac{\nu}{\nu+d}Z_{0}^{-\frac{2}{\nu+d}}\H^T \bbeta,\\
c&=1-\frac{1}{\nu}\|\bmu\|_2^2, \\
\end{split}\label{eq:para}
\end{equation}
where $Z_{0}$ is given in \eqref{eq:Z}.
For convenience, $\nu$ is used in the above formulation.

Meanwhile, one can solve the primal problem in the dual domain by optimizing the following formulation  
\begin{equation}
\begin{split}
\min_{\bbeta}
&\quad\ln_q\left(\exp_q^{r}\left(\frac{r}{2}Z_{0}^{2(1-q)}\|\H^T\bbeta\|_2^2\right)\right) \\
% \ln_q\left((1-\frac{\nu-2}{(\nu+n)^2}Z_{0}^{2(1-q)}\|H^T\bbeta\|_2^2)^{\frac{\nu-2}{2}}\right)\\
% \frac{1-(1-\frac{\nu-2}{(\nu+n)^2}Z_{0}^{2(1-q)}\|H^T\bbeta\|_2^2)^{\frac{\nu-2}{\nu+n}}}{1-q}\\
&\quad+CD_{1/q^{\prime}}\left(\bbeta/C\|\mathbf{1}_m\right)\\
s.t.& \quad \bbeta\ge0,\\
\end{split}\label{eq:dual2}
\end{equation}
where $r=\frac{2+d(1-q)}{2}$, $\H=(y_1\f_1, y_2\f_2,\ldots, y_m\f_m)^T$ and $Z_{0}$ is given in \eqref{eq:Z}. For $q=1$, it becomes the following $\ell_2$-norm regularized problem
\begin{eqnarray*}
\min_{\bbeta}&&\frac{1}{2}\|\H^T\bbeta\|_2^2+CD_{1/q^{\prime}}\left(\bbeta/C\|\mathbf{1}_m\right)\\
s.t. && \bbeta\ge\0.\\
\end{eqnarray*}
% The cases $q^{\prime}\rightarrow 0^+$ and $q^{\prime}\rightarrow 1^-$ is similar to Theorem \ref{thm:1}.
\end{theorem}
Similar to $p_0(\w)$, $p(\w)$ in \eqref{eq:pw} is also a Student-t distribution.
The variance $c$ is decided uniquely by $\bmu$ and $\nu$.
%where the nonlinear parameter $\nu$ influences the posterior distribution via influencing the linear parameter $c$. 
Optimizing $p(\w)$ is equivalent to updating the paramaters $\bmu$ of $p_0(\w)$ according to \eqref{eq:para}, which generalizes the conjugate prior property of exponential family. In the dual formulation \eqref{eq:dual2}, one can see that \eqref{eq:dual2} generalizes the dual formulation of $\ell_2$-norm regularizer by imposing an outer function on $\|\H^T\bbeta\|_2^2$.

Based on the solution \eqref{eq:pw} of $p(\w)$, one can also solve the primal problem directly by simplifying $D_q(p(\w)\|p_0(\w))$. For simplicity, we use $\nu$ instead of $q$ in the following results.

\begin{theorem}\label{thm:3}
For $\nu>0$ and $0\le q^{\prime}\le1$, $p_0(\w)$ and $p(\w)$
are given in \eqref{eq:p0nu}, \eqref{eq:pw} respectively. One can also directly solve NITM by optimizing the following problem
\begin{equation}
\begin{split}
\min_{\bm{\mu}}&\quad\frac{1}{2}\left(1-\frac{1}{\nu}\|\bmu\|_2^2\right)^{-\frac{d}{\nu+d}}\left(\frac{\nu-d}{\nu}\|\bm{\mu}\|_2^2+\nu+d\right)\\
&\quad-\frac{\nu+d}{2}+C\sum_{i=1}^{m}\exp_{q^{\prime}}(-z_i) \\
s.t.  &\quad z_i=\frac{\nu^{\frac{\nu}{\nu+d}}\pi^{-\frac{d}{\nu+d}}}{\nu+d}\Big(\frac{\Gamma(\frac{\nu+d}{2})}{\Gamma(\frac{\nu}{2})}\Big)^{\frac{2}{\nu+d}}\\
&\qquad\qquad\cdot\left(1-\frac{1}{\nu}\|\bmu\|_2^2\right)^{-\frac{d}{\nu+d}}y_i\f_i^T\bm{\mu},\\
&\quad\qquad\qquad\qquad\qquad \text{for}\;i=1,2,\ldots,m, 
\end{split}\label{eq:primal}
\end{equation}
where $\bmu$ is the posterior expectation of $\w$. 
\end{theorem}
From Theorem \ref{thm:3}, minimizing Tsallis divergence from $p(\w)$ to $p_0(\w)$ \emph{w.r.t.} $p(\w)$ is equivalent to minimizing a convex numerical optimization problem $\emph{w.r.t.}$ $\bmu$. 

Summarizing the above results, NITM unifies $\ell_{0/1}$ loss, hinge loss, squared hinge loss and exponential loss by \emph{unnormalized Tsallis entroy} with single parameter $q^{\prime}$. Meanwhile, NITM unifies Gaussian prior and Student-t prior by \emph{Tsallis divergence} and $q$-expectation with single parameter $q$. Furthermore, NITM unifies loss function and regularization by the concepts of NIT. 
In \cite{martins2009nonextensive}, the authors showed nonextensive information theory can also be used in the design of kernel, named \emph{Tsallis mutual information kernel}. By this framework, they unifies the existed linear kernel, Jensen-Shannon kernel and boolean kernel in one parametric family.  Therefore, we have Proposition 1.

\begin{proposition}
All the three parts loss function, regularization and data transform of discriminant model can described consistently by nonextensive information theory.
\end{proposition}

Unlike MaxEnDNet in \cite{zhu2009maximum} which needs resort to variational approximation, we can directly optimize the dual formulation \eqref{eq:dual2} or the primal formulation \eqref{eq:primal} based on gradient-based convex optimization. After optimizing $\bbeta$ in \eqref{eq:dual2} or $\bmu$ in \eqref{eq:primal},  the posterior distribution $p(\w)$ can be acquired in \eqref{eq:pw}.

% While in our work, we show that $\ell_{0/1}$ loss, hinges loss, squared hinge loss and exponential loss can also be unified in a similar framework. In addition, by using Tsallis divergence and $q$-expectation, the computability of exponential family can be transplanted to $q$-exponential family without adding unmanageable complexity.

% The $3$ different parts can be unified in the same information theoretical framework, which may be surprising. We can explain it in two respects. Firstly, consider computability and Occam's Razor rule, convex model often consists of simple polynomials or exponentials. While $q$-exponent and $q$ logarithm are good tools to connect polynomial, exponent and logarithm. Secondly, 
% every optimization model must specify some measure of similarity. In nonnegative regions, the divergence derived from information theory provides good choice. The axiomatic derivation of divergence can be seen in \cite{csiszar1991least}.

\begin{figure*}[!h]
\centering
\includegraphics[scale=0.7]{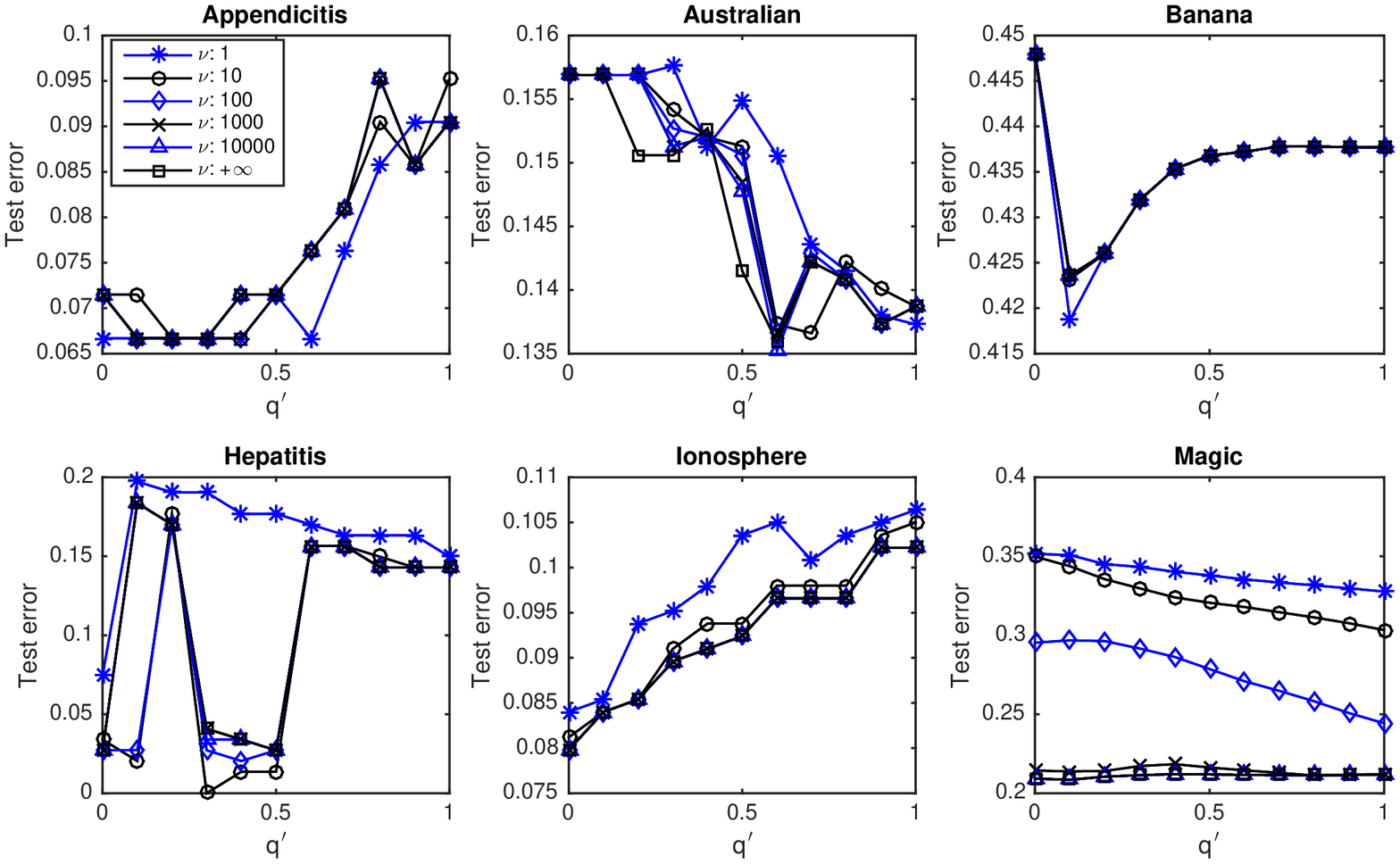}
\caption{Test error \emph{vs.} $q^{\prime}$ for binary classification task}
\label{fg:1}
 \end{figure*}
\section{Experiments}\label{experiments}
We illustrate the performance of NITM with Student-t prior \eqref{eq:prior} on standard datasets. The concrete settings are
\begin{itemize}
\item $6$ standard datasets: appendicitis, australian, banana, hepatitis, ionosphere, magic\footnote{Available at \url{http://keel.es/datasets.php}}. Each dataset is divided into $10$ parts by distribution optimally balanced stratified cross-validation (DOB-SCV) (see \cite{moreno2012study} and reference therein). $3$ parts of them are used as test dataset, while the other $7$ parts are used in cross validation. 
\item Feature transform: for nominal features, we transform them into double values according to their number which starts from $1$. Before learning, all the features are normalized with $0$ mean and unit length. In addition, a column with all $1$ are added to the feature matrix to learn a bias parameter. In this paper, the main interest is the influence of regularization and loss function to empirical generalization performance, so the group of basis functions $\{\phi_i(\cdot)\}_{i=1}^d$ are set as identity matrix.
\item Parameter setting: NITM has $3$ parameters, $\nu, q^{\prime}$ and $C$. Since NITM includes the existing hinge loss-based SVM, squared hinge loss-based SVM and exponential loss-bass classifier as special cases, in this paper NITM is treated as a meta model. Instances of NITM with concrete values of pair $(\nu,q^{\prime})$ are treated as different models. Meanwhile, $C$ is treated as an inner hyperparameter of model. For an instance of NITM with given $(\nu,q^{\prime})$, $C$ is selected by $7$-cross validation on the divided $7$ parts of each dataset. Then instances of NITM with selected $C$ are compared by test error on the rest uninfluenced $3$ parts. In experiments, we compare $66$ models with $\nu$ from $\{1,10,10^2,10^3,10^4,+\infty\}$ and $q^{\prime}$ from $\{0,0.1,0.2,0.3,0.4,0.5,0.6,0.7,0.8,0.9,1\}$. The inner hyperparameter $C$ of each model is selected among $\{1,10^2,10^4,10^6,10^8,10^{10}\}$. 
\item Algorithms: In experiments, we mainly explore the primal convex optimization method to solve NITM. For the model with $q^{\prime}>0$, the optimization problem in \eqref{eq:primal} is smooth, and thus BFGS method is employed. For $q^{\prime}=0$, which corresponds to hinge loss, the optimization problem is nonsmooth, therefore subgradient BFGS method \cite{yu2008quasi} is employed. In addition, backtracking line search is used to get global solution and speeds up the iteration. For each problem, the iteration will be stopped if the number of iterations exceeds $5000$ or the direction vector is orthogonal with gradient vector.
\item Result representation: The result is represented in Fig. \ref{fg:1}. Each subfigure corresponds to a dataset and reflects the test error as $q^{\prime}$ changes. It deserves to note that for each pair $(\nu,q^{\prime})$, $C$ has been selected in the cross validation stage, so the parameter $C$ of each curve is different in general. The legend on the upper left subfigure is shared among the $6$ subfigures.
\end{itemize}
% Then we give analysis for each dataset. 
In Fig. \ref{fg:1}, it is showed that the curves in each subfigure are quite different, which reflects the different physical characteristics of datasets. In order to explain the role of $\nu,q^{\prime}$, the result is analyzed by the order of datasets.
\begin{itemize}
\item Appendicitis: It is showed that better performance is acquired when $q^{\prime}$ is relatively small ($q\in\{0,0.1,0.2,0.3,0.4\}$). However, in general, $\nu$ has little influence on test error, except for quite small $\nu=1$, which get smaller test error for many $q^{\prime}\in\{0.6,0.7,0.8\}$ comparing with other values of $\nu$.
\item Australian: In general, test error will be small if $q^{\prime}$ is large. For $q^{\prime}=0.6$, the best performance is acquired. For fixed $q^{\prime}$, a large $\nu$ is preferred.
\item Banana: The best result will be got when $q^{\prime}=0.1$. Meanwhile, although $\nu$ has little influence, the best test error is got when $\nu=1$.
\item Hepatitis: This dataset prefer middle value of $\nu$, e.g, $\nu=10$. The test error will be $0$ if $\nu=10,q^{\prime}=0.3$. In addition, the curve with $\nu=1$ has different shape from that with other $\nu$'s.
\item Ionosphere: The consistent shape of the $6$ curves shows that this dataset prefer small $q^{\prime}$ and large $\nu$. The best result is got in the case with $q^{\prime}=0, \nu\rightarrow+\infty$, which corresponds to standard hinge loss-based SVM.
\item Magic: This consistent shape shows that large $q^{\prime}$ and large $\nu$ is preferred. Then best result is acquired when $q^{\prime}=1.0,\nu\rightarrow+\infty$, which corresponds to exponential loss with $\ell_2$ regularization.
\end{itemize}
The result shows that different datasets prefer different settings of $(\nu,q^{\prime})$ which is a verification of no-free lunch theorem \cite{wolpert2002supervised}. Although the result seems to be disorder, it is showed that compared with only tuing $C$, tuning $q^{\prime},\nu,C$ independently is not equivalent to tuning $C$ only and can give extra gain of generalization performance.

\section{Proofs}\label{proofs}
\subsection{Proof of Theorem 2}
\begin{proof}
As we say, for $q\ge1$ and $0\le q^{\prime}\le 1$,  the general problem is a convex program.
The Lagrangian associated with the general model is 
\begin{eqnarray*}
&&\cL(p(\w),\z,\beta_0,\bbeta) \\
&=&D_q(p(\w)\|p_0(\w))+C\sum_{i=1}^{m}[1-(1-q^{\prime})z_i]_{+}^{\frac{1}{1-q^{\prime}}}\\ 
%\frac{\int p^q(\w) p^{1-q}_0(\w)d\w -1}{q-1}\\
&&+\beta_0\left(\int p(\w)d\w-1\right)\\
&&+\sum_{i=1}^{m}\beta_i\left(z_i-\int y_i \f_i^T\w p^q(\w)d\w\right)
\end{eqnarray*}
The Lagrangian dual function is defined as $\cL^*(\beta_0,\bbeta)=\inf_{p(\w),\z}\cL(p(\w),\z,\beta_0,\bbeta)$. 
Denote $\H=(y_1\f_1, y_2\f_2,\ldots, y_m\f_m)^T$.
For $q>1$, taking the variational derivative of $\cL$ \emph{w.r.t.} $p$, one gets
\begin{eqnarray*}
\frac{\partial \cL}{\partial p}&=&\frac{q}{q-1}\left(\frac{p}{p_0}\right)^{q-1}+\beta_0-\bbeta^T \H\w  \cdot qp^{q-1}
\end{eqnarray*}
Setting the variational derivative to $0$, one has the following expression,
\begin{eqnarray}
p(\w)
&=&\frac{1}{Z_q}\left[p_0^{1-q}(\w)+(1-q)\bbeta^T \H\w\right]_{+}^{\frac{1}{1-q}},\nonumber\\\nonumber
&=&\frac{1}{Z_q}p_0(\w)\exp_q(p_0^{q-1}(\w)\bbeta^T \H\w)^{\frac{1}{1-q}}
\end{eqnarray}
which uses Tsallis cut-off prescription \cite{teweldeberhan2005cut} for $1+(1-q)p_0^{q-1}(\w)\bbeta^T \H\w<0$ and $Z_q=\int p_0(\w)\exp_q(p_0^{q-1}(\w)\bbeta^T \H\w)^{\frac{1}{1-q}}d\w$ is a normalization constant and $\beta_0=\frac{qZ_q^{q-1}}{1-q}$.

For $q=1$, similarly one gets
\begin{eqnarray*}
\frac{\partial \cL}{\partial p}&=&1+\ln\frac{p}{p_0}+\beta_0-\bbeta^T \H\w
\end{eqnarray*}
Setting the derivative to $0$, one has 
\[p(\w)=\frac{1}{Z_1}p_0(\w)\exp(\
\bbeta^T \H\w),
\]
where $Z_1=\int p_0(\w)\exp(\
\bbeta^T \H\w)d\w$ and $\beta_0=-1+\ln Z_1$.

For $0<q^{\prime}<1$, substituting $p(\w)$ and $\beta_0$ into $\cL$, one has
\begin{eqnarray*}
&&\cL^{*}(\bbeta) \\
&=&\inf_{p(\w),\z}\cL(p(\w),\z,\beta_0,\bbeta)\\	
&=&-\ln_{q}(Z_q(\bbeta))+C\frac{\sum_{i=1}^{m}(q^{\prime}(\frac{\bbeta_i}{C})^{\frac{1}{q^{\prime}}}-\frac{\bbeta_i}{C})}{q^{\prime}-1}\\
&&+\sum_{i=1}^m I_{\infty}(\beta_i\ge0),\\
&=&-\ln_{q}(Z_q(\bbeta))-CD_{1/q^{\prime}}(\bbeta/C\|\1_m)+Cm
\end{eqnarray*}

Similarly, for $q^{\prime}\rightarrow 0$ and $q^{\prime}\rightarrow 1$, 
\begin{eqnarray*}
&&\cL^{*}(\bbeta) \\
&=&\inf_{p(\w),\z}\cL(p(\w),\z,\beta_0,\bbeta)\\	
&=&-\ln_{q}(Z_q(\bbeta))+\sum_{i=0}^{m}\beta_i\\
&&+\sum_{i=1}^m I_{\infty}(0\le\beta_i \le C),\\
&=&-\ln_{q}(Z_q(\bbeta))-CD_{\infty}(\bbeta/C\|\1_m)+Cm
\end{eqnarray*}
and
\begin{eqnarray*}
&&\cL^{*}(\bbeta) \\
&=&\inf_{p(\w),\z}\cL(p(\w),\z,\beta_0,\bbeta)\\	
&=&-\ln_{q}(Z_q(\bbeta))+C\sum_{i=1}^{m}(\frac{\beta_i}{C}-\frac{\beta_i}{C}\ln\frac{\beta}{C})\\
&&+\sum_{i=1}^m I_{\infty}(\beta_i\ge0),\\
&=&-\ln_{q}(Z_q(\bbeta))-CD_{1}(\bbeta/C\|\1_m)+Cm,
\end{eqnarray*}
where $I(\cdot)$ is an indicator function defined in Section 2. 

Neglecting constant $Cm$, Theorem 2 is proved.
\end{proof}
\subsection{Proof of Theorem 3}
\begin{proof}
Impose the prior distribution (9) and set $\frac{1}{q-1}=\frac{\nu+d}{2}$, then
\begin{eqnarray*}
&&p_0^{1-q}(\w)+(1-q)\bbeta^T \H\w \\
&=&\frac{1}{Z_{0}^{1-q}}\left(1
+\frac{1}{\nu}\w^T\w\right)-\frac{2}{\nu+d}\bbeta^T \H\w\\
&=&\frac{1}{Z_{0}^{1-q}}\big(c
+\frac{1}{\nu}
\left\|\w-\bmu\right\|_2^2\big)
\end{eqnarray*}
where 
\begin{eqnarray*}
\bmu&=&\frac{\nu}{\nu+d}Z_{0}^{-\frac{2}{\nu+d}}\H^T \bbeta,\\
c&=&1-\frac{1}{\nu}\|\bmu\|_2^2\\
&=&1-\frac{\nu}{(\nu+d)^2}Z_{0}^{-\frac{4}{\nu+d}} \|\H^T \bbeta\|_2^2.
\end{eqnarray*}
Then if $c<0$,
\begin{eqnarray*}
&&p(\w)\\
&=&\frac{1}{Z_q}\left[p_0^{1-q}+(1-q)\bbeta^T \H\w\right]_{+}^{\frac{1}{1-q}}\\
&=&\frac{1}{Z_q Z_{0}}(-c)^{\frac{1}{1-q}}\left[\frac{1}{-\nu c}\|\w-\bmu\|_2^2-1\right]_+^{\frac{1}{1-q}}
\end{eqnarray*}
By our setting, $\frac{1}{1-q}=-\frac{\nu+d}{2}\le-\frac{d}{2}<-\frac{1}{2}$. Then 
% For $r>\frac{1}{2}$, by the fact, 
% \begin{eqnarray*}
% &&\int_1^{\infty}\frac{1}{(x^2-1)^{r}}dx\\
% &=&\int_0^{\infty}\frac{1}{2}t^{-r}(t+1)^{-\frac{1}{2}}dt\\
% &\ge&\int_0^{1}\frac{1}{2\sqrt{2}}t^{-r}dt\\
% &=&\frac{1}{2\sqrt{2}}\frac{1}{-r+1}t^{-r+1}\Big|_{0^+}^1\\
% &\rightarrow&+\infty.
% \end{eqnarray*}
 $p(\w)$ is unnormalizable and do not satisfy the constraint $\int p(\w) d\w=1$. Similarly, if $c=0$, $p(\w)$ is also not unnormalizable. 
Therefore, in our setting, $c>0$. 
Then we have
\begin{eqnarray*}
p(\w)&=&\frac{1}{Z_q Z_{0}}c^{\frac{1}{1-q}}\left[1+\frac{1}{\nu c}\|\w-\bmu\|_2^2\right]^{\frac{1}{1-q}}
\end{eqnarray*}
From the fact 
\begin{eqnarray*}
\int \frac{1}{Z_{0}c^{d/2}}\left[1+\frac{1}{\nu c}\|\w-\bmu\|_2^2\right]^{\frac{1}{1-q}}d\w=1.
\end{eqnarray*}
and $\int p(\w)d\w=1$, it follows that
\begin{eqnarray}
&&Z_q=c^{\frac{1}{1-q}+\frac{d}{2}}=c^{-\frac{\nu}{2}}\nonumber\\
&=&\left(1-\frac{\nu}{(\nu+d)^2}Z_{0}^{-\frac{4}{\nu+d}} \|\H^T \bbeta\|_2^2\right)^{-\frac{\nu}{2}} \label{eq: zqc}\nonumber\\
% &=&\left(1+\frac{(1-q)(2+d(1-q))}{4}Z_{0}^{2(1-q)}\|\H^T \bbeta\|_2^2\right)^{\frac{2+d(1-q)}{2(1-q)}}\nonumber\\
&=&\exp_q^r \left(\frac{r}{2}Z_{0}^{2(1-q)}\|\H^T \bbeta\|_2^2\right),
\end{eqnarray}
where $r=\frac{2+d(1-q)}{2}$, $Z_{0}$ is given in (12).

Substituting \eqref{eq: zqc} into Theorem 2 and simplifying the case $q\rightarrow  1 (\nu\rightarrow +\infty)$, Theorem 3 is proved.
\end{proof}
\subsection{Proof of Theorem 4}
\begin{proof}
From the Proof of Theorem 3,
\begin{eqnarray*}
p(\w)=\frac{1}{Z_{0}c^{d/2}}\left[1+\frac{1}{(\nu-2)c}\|\w-\bmu\|_2^2\right]^{\frac{1}{1-q}},
\end{eqnarray*}
where 
\begin{equation}
c=1-\frac{1}{\nu-2}\|\mu\|_2^2.
\end{equation}
Use the definition of Tsallis divergence and $\frac{1}{q-1}=\frac{\nu+d}{2}$, we can get
\begin{equation*}
\begin{split}
&D_q(p(\w)\|p_0(\w))\\
&=\frac{1}{2}\left(1-\frac{1}{\nu}\|\bmu\|_2^2\right)^{-\frac{d}{\nu+d}}\left(\frac{\nu-d}{\nu}\|\bm{\mu}\|_2^2+\nu+d\right)
-\frac{\nu+d}{2}
\end{split}
\end{equation*}
Use the formulation of normalized Student t distribution, one can compute the constraint (4) as
\begin{eqnarray*}
z_i&=&\frac{\nu^{\frac{\nu}{\nu+d}}\pi^{-\frac{d}{\nu+d}}}{\nu+d}\Big(\frac{\Gamma(\frac{\nu+d}{2})}{\Gamma(\frac{\nu}{2})}\Big)^{\frac{2}{\nu+d}}\\
&&\cdot\left(1-\frac{1}{\nu}\|\bmu\|_2^2\right)^{-\frac{d}{\nu+d}}y_i\f_i^T\bm{\mu},\\
&&\quad\qquad\qquad\qquad\qquad \text{for}\;i=1,2,\ldots,m.
\end{eqnarray*}

Substituting it into the general model 3, Theorem 4 is proved.

% \begin{eqnarray*}
% &&\int p^q(\w)p_0^{1-q}(\w)d\w\\
% &=&\int\frac{1}{Z_{0}c^{\frac{nq}{2}}}
% \end{eqnarray*}
% \end{proof}
% \subsection{Proof of Theorem 4}
% \begin{proof}

\end{proof}

\section{Conclusion and future work}\label{conclusion}
In this paper, we proposed a new discriminant model named nonextensive information theoretical machine (NITM) based on nonextensive information theory. NITM gives a consistent view of regularization and loss function and takes $\ell_{0/1}$, hinge loss, squared hinge loss and exponential loss as special cases. The solution and explicit primal and dual formulations are given. Then experiments show the improvement of generalization performance by tuning $\nu,q^{\prime}$.
% \appendix
\bibliography{ML}
\bibliographystyle{icml2016}

\end{document}